\documentclass[10pt,letterpaper]{article}

\usepackage[T1]{fontenc}
\usepackage[utf8]{inputenc}
\usepackage{dropd_preprint}
\usepackage{xcolor}
\usepackage[authoryear,round]{natbib}
\setcitestyle{citesep={;},aysep={,},yysep={;}}
\usepackage{amsmath,amssymb,amsthm,mathtools}
\usepackage{microtype}
\usepackage{booktabs}
\usepackage{adjustbox}
\usepackage{graphicx}
\usepackage{fontawesome5}
\usepackage{subcaption}  
\usepackage{algorithm,algpseudocode}
\usepackage{placeins}

\theoremstyle{plain}
\newtheorem{theorem}{Theorem}
\newtheorem{proposition}[theorem]{Proposition}
\theoremstyle{definition}
\newtheorem{definition}[theorem]{Definition}
\newtheorem{assumption}[theorem]{Assumption}
\theoremstyle{remark}

\usepackage{enumitem}

\newcommand{\EE}{\mathbb E}
\DeclareMathOperator{\KL}{KL}
\DeclareMathOperator{\clip}{clip}
\DeclareMathOperator*{\argmin}{arg\,min}

\usepackage[most]{tcolorbox}      \definecolor{promptbg}{HTML}{F7F7F9}
\definecolor{promptfr}{HTML}{C7C7CE}
\newtcblisting{promptbox}{breakable, enhanced, listing only,
  colback=promptbg, colframe=promptfr, boxrule=0.5pt, arc=2pt,
  left=6pt, right=6pt, top=4pt, bottom=4pt, toptitle=1pt, bottomtitle=1pt,
  fonttitle=\footnotesize\bfseries,
  coltitle=black, colbacktitle=promptbg, borderline horizontal={0.5pt}{0pt}{promptfr},
  listing options={basicstyle=\ttfamily\scriptsize, breaklines=true,
                   breakindent=0pt, columns=fullflexible, keepspaces=true,
                   showstringspaces=false, aboveskip=0pt, belowskip=0pt}}

\DeclareRobustCommand{\revision}[1]{#1}
\usepackage{arydshln}
\usepackage{colortbl}

\definecolor{dropdrow}{RGB}{234,245,253}   \definecolor{dropdblue}{RGB}{45,149,233}   \definecolor{dropdwarm}{RGB}{253,93,64}    \par
\newcommand{\gain}[1]{\textsubscript{\textcolor{dropdwarm}{+#1}}}
\newcommand{\loss}[1]{\textsubscript{\textcolor{dropdblue}{-#1}}}

\title{Dr. OPD: Learning What to Follow for Optimal On-Policy Distillation of Large Language Models}
\author{Zhenyu Wang\thanks{Equal contribution.}\quad
  Tianze Wang\footnotemark[1]\quad
  Linjun Zhang\thanks{Co-corresponding authors.}\quad
  Yifan Hu\footnotemark[2]\\[0.5em]
  \normalsize Department of Statistics, Rutgers University\\[0.35em]
  \normalfont\small\texttt{\{zw425,tw522\}@stat.rutgers.edu}\\
  \normalfont\small\texttt{zlj11112222@gmail.com}\quad
  \texttt{yifan.hu@rutgers.edu}
}

\date{}

\usepackage{hyperref}
\hypersetup{
  colorlinks=true,
  citecolor=magenta,
  filecolor=blue,
  linkcolor=blue,
  urlcolor=blue,
  pdftitle={Dr. OPD: Learning What to Follow for Optimal On-Policy Distillation of Large Language Models},
  pdfauthor={Zhenyu Wang, Tianze Wang, Linjun Zhang, Yifan Hu},
  pdfsubject={On-policy distillation of large language models}
}

\begin{document}
\maketitle
\makeatletter
\begingroup
\renewcommand{\thefootnote}{}
\let\@footnotetext\H@@footnotetext
\footnotetext{\faGithub\enspace Code: \url{https://github.com/zywang0701/Dr-OPD}.}
\endgroup
\makeatother

\begin{abstract}
On-policy distillation (OPD) trains a student on its own generated responses using dense, token-level supervision from a stronger teacher. Vanilla OPD treats all teacher signals equally, assuming that the teacher's supervision is equally important for every token. However, teacher signals at different tokens may have very different effects on the student's performance: some correct important reasoning errors, while others have little effect on the final answer. 
Motivated by this observation, we introduce \textbf{Dr. OPD} (\emph{OPD Done Right}), which defines the optimal weighted OPD to maximize the student's performance.
We formulate Dr. OPD as a bilevel optimization problem in which the student learns from weighted teacher supervision, while the weights are selected to maximize the expected reward of the resulting student. To solve Dr. OPD, we develop an efficient iterative solver that updates the token weights and student policy alternatively. At each round, it updates weights in closed form and then takes one gradient step on the resulting weighted OPD objective. Under regularity conditions, we show that this weighted update achieves a higher expected reward than a vanilla OPD update. Empirically, across strong-to-weak and same-size distillation on math and code, Dr. OPD consistently outperforms all evaluated baselines. In particular, in the strong-to-weak distillation setting, Dr. OPD improves average math performance by $9.7$ points over vanilla OPD, and enables the smaller student to surpass its larger teacher.
\end{abstract}

\section{Introduction}
\label{sec:introduction}

On-policy distillation (OPD) has become an important post-training approach for transferring capabilities from a stronger teacher to a student model that is cheaper to deploy or has undergone less training
\citep{agarwal2024onpolicydistillationlanguagemodels,deepseekai2026deepseekv4,kimiteam2026kimik3}. Unlike outcome-reward reinforcement learning methods like RLOO and GRPO \citep{ahmadian2024rloo,shao2024deepseekmathpushinglimitsmathematical}, which update the model based on a scalar reward for an entire generated trajectory, OPD leverages the teacher model to provide dense, token-level supervision along with the student's generated responses \citep{gu2026minillmonpolicydistillationlarge}.
Specifically, OPD trains the student by minimizing the divergence between the next-token distributions of the student and the teacher at each token visited by the student. Therefore, it provides a learning signal at every generated token rather than only at the end of a response.

Vanilla OPD assigns equal importance to the teacher's supervisions across tokens. 
However, given a long response generated by the student, should every teacher signal be treated equally? Consider a student solving a math problem. The teacher could either correct a crucial reasoning step or simply prefer different wording under the same mathematical meaning. These two signals are supposed to enter the distillation loss with different weights, as their effects on the final answer can be very different. Indeed, recent work has likewise shown that not every signal  from the teacher across each token is equally useful to the student \citep{armandpour2026unmasking,xie2026positionbias,xing2026tropd}.

To study how the student should trust teacher signals across tokens, we revisit the goal of OPD. In particular, note that the purpose of distillation is ultimately to improve the student's task performance. Minimizing the divergence towards the stronger teacher policy is only an approach used to achieve this goal.
Indeed, the teacher signal only tells us what the teacher prefers at each token, but not whether following the preference improves the student's performance. This motivates us to adaptively choose the weight of each teacher signal according to its usefulness for the student. In this paper, we therefore study how to weight token-level teacher signals so that the student achieves the best performance.

Following this principle, we introduce \textbf{Dr. OPD} (\emph{OPD Done Right}), which assigns weights $w(s,v)$ to each teacher signal on the tokens $v$ and context $s = (x,y)$ according to its usefulness for improving the student, as illustrated in Figure~\ref{fig:dr_opd}. A weight $w_t>1$ places more trust in the teacher at token $t$, while $0<w_t<1$ places less trust. Vanilla OPD sets equal weights of $w_t\equiv 1$ across all tokens.
We note that for any positive weight $w$, the corresponding $w$-weighted OPD remains a valid distillation objective that pushes the student toward the teacher policy, just as vanilla OPD does. Dr. OPD then selects, among these valid weights, the one that gives the best resulting student:
\begin{tcolorbox}[
  colback=dropdrow, colframe=dropdblue!65!white,
  boxrule=0.6pt, arc=2pt,
  left=8pt, right=8pt, top=2pt, bottom=2pt,
  before skip=10pt, after skip=10pt
]
\begin{equation}
    \max_{w}
    \Bigl\{
    \textrm{Student's performance after training with $w$-weighted OPD}
    \Bigr\}.
    \label{eq:dr_opd_draft}
\end{equation}
\end{tcolorbox}
The performance is measured by an outcome reward based on its generated responses~\citep{ouyang2022training,shao2024deepseekmathpushinglimitsmathematical}.
We formally formulate Dr. OPD as a bilevel optimization in Section~\ref{sec:dr_opd}.

\begin{figure}[t]
    \centering
\includegraphics[width=\linewidth,trim=0 70bp 0 65bp,clip]{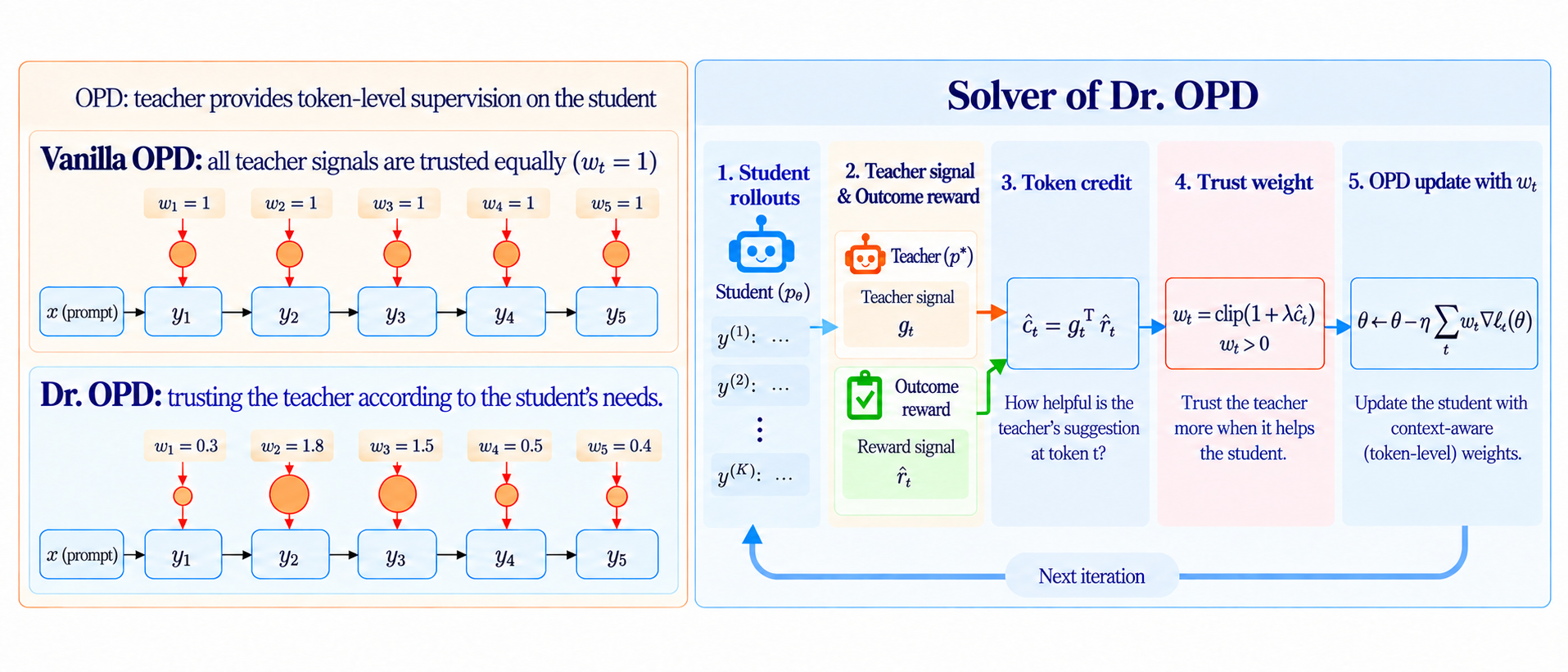}
    \caption{\textbf{Illustration of Dr. OPD.}
Left: vanilla OPD assigns the same weight to every teacher signal, while Dr. OPD adapts the weights according to their benefit to the student.
Right: student rollouts receive teacher and outcome-reward signals, whose gradient alignment determines token credits.
These credits are converted into weights for distillation, where $\lambda$ controls the weighting strength.
}
    \label{fig:dr_opd}
\end{figure}

Solving \eqref{eq:dr_opd_draft} exactly would require repeatedly training the student under different weight functions and then selecting the best weight, which is computationally infeasible. We therefore develop an iterative solver, illustrated in Figure~\ref{fig:dr_opd}. At each round, we build a surrogate objective of \eqref{eq:dr_opd_draft} to update the weights, leading to a simple closed-form update instead of learning a complex weight function. Then the student performs a weighted OPD update. Specifically, the student generates its own responses, from which we obtain the teacher's token-level supervision and the outcome reward signal. We use these two signals to assess whether following the teacher more closely at each token would improve the student's performance. Under regularity conditions, we show that the student, after the weighted update, achieves a higher reward than the corresponding vanilla OPD update.

To summarize, we make the following contributions:
\begin{itemize}[leftmargin=1.5em, topsep=1pt, itemsep=1pt]
\item \textbf{Bilevel Formulation of Distillation.}
We introduce the concept of learning useful signals from the teacher via weighted OPD and propose Dr. OPD, a bilevel optimization formulation that defines the optimal weighted OPD to maximize the student's performance. The lower-level problem trains the student under a given weight function, while the upper-level problem selects the weights according to the performance of the resulting student.

\item \textbf{Efficient Algorithm.}
Despite the challenge of solving Dr. OPD directly, we develop an efficient iterative solver with a closed-form weight update followed by a gradient update of the weighted OPD. 
To incorporate outcome rewards into the weight updates, it requires computing an inner product between two parameter-dimensional gradients for every token, which is prohibitively expensive for large models. We instead express each inner product as the product of two scalars, and then these products for all tokens can be computed altogether using a single Jacobian--vector product (JVP); see Section~\ref{sec:methodology} for details.

\item \textbf{Significant Empirical Gains.}
We evaluate Dr. OPD across strong-to-weak and same-size distillation on math and code tasks, where it achieves the best  performance across all evaluated settings and consistently outperforms all baselines. As shown in Figure \ref{fig:gains}, Dr. OPD improves Qwen3-1.7B over vanilla OPD by $+9.7$ points on math and even surpasses the larger Qwen3-4B teacher (35.0 vs.\ 34.1). In same-size distillation, it also surpasses both RL-trained teachers.
\end{itemize}
\begin{figure}[!ht]
    \centering
    \includegraphics[width=\linewidth,trim=0 123bp 0 76bp,clip]{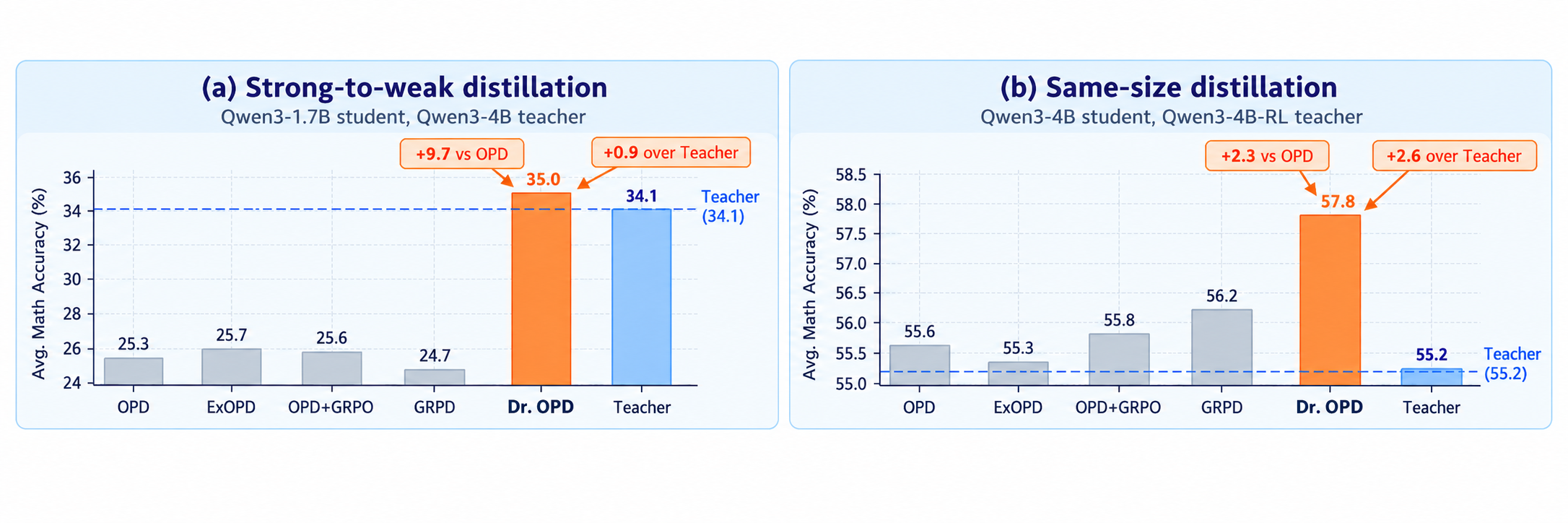}
    \caption{\textbf{Empirical gains.}
Across both distillation setups, Dr. OPD consistently outperforms existing baselines and even produces a student that surpasses its teacher (Full results in Section \ref{sec:exp-results}).}
    \label{fig:gains}

\end{figure}

\section{Weighted On-Policy Distillation}
\label{sec:formulation}

OPD trains a student model on its generated responses under the supervision from a stronger frozen teacher model. Let $x$ be a prompt, and $y=(y_1, \ldots, y_T)$ a response generated by the student. At step $t$, the context $s_t=(x, y_{<t})$ contains the prompt and all the preceding tokens generated before $t$. We denote $p_\theta(v\mid s_t)$ as the student's probability of generating the next token $v\in \mathcal{V}$, and $p^*(v\mid s_t)$ as the teacher's probability in the same context, where $\mathcal{V}$ is the vocabulary of tokens.

Vanilla OPD aims to train the student policy to mimic the teacher policy using the reverse KL divergence:
\begin{equation}
    \min_\theta D(\theta):= \EE\left[\sum_{t} \KL \left(p_\theta(\cdot \mid s_t)\,\|\, p^*(\cdot \mid s_t)\right)\right],
    \label{eq:plain_opd}
\end{equation}
where 
\[
\KL \left(p_\theta(\cdot \mid s)\,\|\, p^*(\cdot \mid s)\right)=\sum_v p_\theta(v\mid s)\log \frac{p_\theta (v\mid s)}{p^*(v\mid s)},
\]
and the expectation is taken over the prompts and responses generated by the student policy. Following prior works \citep{agarwal2024onpolicydistillationlanguagemodels,shao2026tokenlevelanalysissampledtokenreversekl}, the generated contexts are treated as fixed during each student update and are refreshed as the student changes.

As shown in \eqref{eq:plain_opd}, vanilla OPD assigns equal weights to teacher signals at every token. However, not all teacher signals are equally informative \citep{armandpour2026unmasking,xie2026positionbias}. 
We introduce a weight function $w(s,v)> 0$, which controls the importance of matching the teacher's probability of token $v$ at context $s$. For a fixed context $s$, we use $w(s,\cdot)$ to denote the corresponding vector of token weights. We then define a weighted KL divergence that generalizes the ordinary KL divergence.

\begin{definition}[Weighted KL divergence]
    \label{def:weighted_kl}
    At a context $s$, given the positive weights $w(s,\cdot)$, we define 
    \begin{align} 
    &\KL_{w(s,\cdot)} \!\left( p_\theta(\cdot\mid s) \,\|\, p^*(\cdot\mid s) \right) \nonumber\\ &\qquad:= \sum_v w(s,v) \left[ p_\theta(v\mid s) \log \frac{p_\theta(v\mid s)} {p^*(v\mid s)} - p_\theta(v\mid s) + p^*(v\mid s) \right]. \label{eq:weighted_kl} 
    \end{align}
\end{definition}
where the additional term $- p_\theta(v\mid s) + p^*(v\mid s)$ ensures nonnegativity under token-dependent weighting, following the generalized KL divergence (also known as I-divergence) with nonnegative measures~\citep{csiszar1975divergence,banerjee2005clustering}.
This weighted KL divergence equals zero if and only if $p_\theta(v\mid s)=p^*(v\mid s)$ for every token $v$. 
Moreover, when $w(s,v)\equiv 1$ for all $v$, it holds $\sum_v w(s,v) \bigl(p^*(v\mid s)-p_\theta(v\mid s) \bigr)=0$, so \eqref{eq:weighted_kl} reduces to the ordinary KL divergence. 

With the weighted KL divergence, we define the weighted OPD as follows.
\begin{definition}[Weighted OPD]
    Given a weight function $w(s,\cdot)> 0$ for every context $s$, we define weighted OPD by
    \begin{equation} 
    \min_\theta D_w(\theta) := \mathbb{E}\!\left[ \sum_{t} \KL_{w(s_t,\cdot)} \!\left( p_\theta(\cdot\mid s_t) \,\|\, p^*(\cdot\mid s_t) \right) \right].
    \label{eq:weighted_opd} \end{equation} 
\end{definition}
We next show that this remains a valid distillation objective. In particular, minimizing $D_w(\theta)$, for any positive $w$, still pushes the student to match the teacher.

\begin{theorem}[Validity of weighted OPD]
    \label{thm:weighted_opd_validity}
    For any positive weight function $w$, we have
    \[
    D_w(\theta)\geq 0.
    \]
    Equality holds if and only if $p_\theta(v\mid s) = p^*(v\mid s)$ for every token $v$ at almost every context under the reference context distribution. When $w(s,v)\equiv 1$, weighted OPD in \eqref{eq:weighted_opd} reduces to vanilla OPD.
\end{theorem}

\section{OPD Done Right: Find the Best Weight for OPD}
\label{sec:dr_opd}
We now study how to choose the weight function $w$.
Recall that the ultimate goal of distillation is to improve the student's task performance, while minimizing the divergence to a stronger teacher is merely a means to achieve this goal. This principle motivates us to select the weight function according to the performance of the resulting student, formulating a bilevel optimization problem.

\subsection{Formulation of Dr. OPD}

We evaluate the student's performance using an outcome reward $V(x,y)\in[0,1]$ for each prompt--response pair $(x,y)$. The student's expected reward is then
\begin{equation*}
    R(\theta) := \mathbb{E}_{x,\;y\sim p_\theta(\cdot\mid x)}\bigl[V(x,y)\bigr].
\end{equation*}
We now introduce \textbf{DR. OPD} \emph{(OPD done right)}, the optimal weighted OPD to maximize the student's performance:
\begin{equation}
{
    \max_{w_{\min}\leq w\leq w_{\max}}
    R\!\left(\theta^*(w)\right)
    \qquad
    \text{subject to}
    \qquad
    \theta^*(w) \in\arg\min_\theta D_w(\theta).
}
    \label{eq:bilevel}
\end{equation}
Dr. OPD is a bilevel optimization problem, where the upper level chooses the weight function $w$ to maximize the resulting student's reward, and the lower level trains the student $\theta$ by minimizing $D_w(\theta)$, under the given $w$.
Here, the constraint $w\in [w_{\min}, w_{\max}]$ means pointwise that $w(s,v)\in[w_{\min},w_{\max}]$ for every $(s,v)$, where $0<w_{\min}<1<w_{\max}$ are prespecified bounds. 
For simplicity, we write $\theta^*(w)$ as if the lower-level solution were unique. When multiple lower-level solutions exist, we choose the one that maximizes the upper-level reward; see the formal definition in Appendix~\ref{app:bilevel}.

\subsection{Iterative Solver for Dr. OPD}
\label{sec:local_weighting}

Solving the bilevel Dr. OPD problem in \eqref{eq:bilevel} exactly is computationally intractable, since evaluating each candidate weight function $w$ would require fully optimizing the lower-level student and then evaluating its reward. In addition, the weight function $w(s,v)$ is defined over all possible context--token pairs, making the upper-level optimization infinite-dimensional.

We therefore design an iterative solver that alternates between student and weight updates. At each round, starting from the current student $\bar\theta$, we take one gradient update of weighted OPD with the weights fixed. We then update the weights according to the reward of the resulting updated student, restricting the weights update to the sampled tokens in responses. This reduces the original infinite-dimensional functional optimization into a finite-dimensional problem at each step.

\paragraph{Lower-level update.}
For a fixed weight function $w$, we take one gradient descent step from the current student $\bar{\theta}$. We generate a response $y=(y_1,\ldots,y_T)$ from $p_{\bar{\theta}}$ and hold the sampled trajectory fixed during this round. Here, $y_t$ is the realized token at step $t$, whereas the previous $v\in\mathcal{V}$ in \eqref{eq:plain_opd} indexes all possible tokens. We write
\[
    w_t := w(s_t,y_t),
    \qquad
    p_t(\bar\theta):=p_{\bar\theta}(y_t\mid s_t),
    \qquad
    p_t^*:=p^*(y_t\mid s_t).
\]

Differentiating the weighted OPD objective at $\bar{\theta}$ gives
\begin{equation}
    -\nabla_\theta D_w(\bar{\theta})
    =
    \mathbb{E}\!\left[
        \sum_t w_t g_t
    \right],\quad \textrm{where}\quad 
    g_t
    :=
    \log\frac{p_t^*}{p_t(\bar{\theta})}
    \left.
    \nabla_\theta \log p_\theta(y_t\mid s_t)
    \right|_{\theta=\bar{\theta}}.
    \label{eq:weighted_opd_gradient}
\end{equation}
The derivation is provided in Appendix~\ref{app:weighted_gradient}.
The corresponding one-step update on the lower-level solution is therefore
\begin{equation}
    \theta^+(w)
    :=\bar\theta - \eta\nabla D_w(\bar\theta)=
    \bar{\theta}
    +
    \eta\,
    \mathbb{E}\!\left[
        \sum_t w_t g_t
    \right],
    \label{eq:one_step_update}
\end{equation}
where $\eta>0$ is the student step size. 
The expectation in \eqref{eq:weighted_opd_gradient}--\eqref{eq:one_step_update} is taken over the current policy $p_{\bar\theta}$.

We call $g_t$ in \eqref{eq:weighted_opd_gradient} the teacher correction at token $y_t$, which is interpreted as follows. If the teacher assigns token $y_t$ a higher probability than the student, then $g_t$ increases the student's probability of generating $y_t$. On the other hand, if the teacher assigns it a lower probability, then $g_t$ decreases its probability. The weight $w_t$ controls the strength of this teacher correction to the student update.

\paragraph{Upper-level update.}
We next update the weight function $w$ according to the reward of the updated student $\theta^+(w)$. 
Since the weight function is defined over all possible context–token pairs and we do not impose a parametric form on it, a gradient update like in \eqref{eq:one_step_update} does not apply here. Instead, we construct a tractable lower bound approximation on $R(\theta^+(w))$ and optimize this bound over the weight function. As shown below, this leads to a simple pointwise update for the weights.

We denote the one-step vanilla OPD update from the current student $\bar\theta$ as follows, which sets $w\equiv 1$,
\[
    \theta_{\mathrm{vani}}^+
    :=
    \theta^+(1).
\]
To determine how the weights should be updated, we ask how changing the weight $w_t$ of the teacher signal at token $t$ affects the resulting student's reward. We define the \textbf{credit} at token $t$ as
\begin{equation}
    c_t
    :=
    \frac{1}{\eta}
    \left.
    \frac{\partial R\!\left(\theta^+(w)\right)}
         {\partial w_t}
    \right|_{w\equiv 1}
    =
    g_t^\top
    \nabla R\!\left(\theta_{\mathrm{vani}}^+\right),
    \label{eq:reward_credit}
\end{equation}
where the second equality follows from the chain rule and \eqref{eq:one_step_update}. The derivative is understood per unit reference measure; Appendix~\ref{app:credit} gives its precise functional interpretation.

The credit $c_t$ measures whether following the teacher more strongly at token $t$ would locally improve the student's reward. Recall that the teacher correction $g_t$ in \eqref{eq:weighted_opd_gradient} locally moves the sampled token's probability toward the teacher's. However, moving closer to the teacher does not necessarily improve the student's task performance. The credit $c_t$ captures this distinction. If $c_t>0$, placing more weight on the teacher correction is beneficial to the student's reward, suggesting that we should trust this signal more and choose $w_t>1$. On the other hand, if $c_t<0$, following the teacher more strongly is harmful, suggesting that the signal should be downweighted with $w_t<1$. Thus, $g_t$ determines the direction of the token-level teacher correction, while $c_t$ determines how strongly the student should follow the teacher signal.

We next derive a lower bound surrogate on the reward $R(\theta^+(w))$.
\begin{proposition}[Lower bound on reward]
    \label{prop:reward_lower_bound}
    Suppose that the reward gradient $\nabla R(\theta)$ is locally $L_R$-Lipschitz in $\theta$ and that the regularity conditions in Appendix~\ref{app:reward_bound} hold. 
    Then, for any weight function
    $w\in[w_{\min},w_{\max}]$ and any $\lambda>0$ satisfying $\lambda \eta L_R
        \mathbb{E}\!\left[
            \sum_t \|g_t\|^2
        \right]
        \leq 1.$
    \begin{equation}
        R\!\left(\theta^+(w)\right)
        \geq
        R\!\left(\theta_{\mathrm{vani}}^+\right) +
        \eta\,
        \mathbb{E}\!\left[
            \sum_t
            \left(
                (w_t-1)c_t
                -
                \frac{(w_t-1)^2}{2\lambda}
            \right)
        \right].
        \label{eq:reward_lower_bound}
    \end{equation}
\end{proposition}
The right-hand side provides a tractable lower-bound surrogate for $R(\theta^+(w))$. The parameter $\lambda>0$ is introduced to simplify the lower bound; see Appendix~\ref{app:reward_bound} for details.

Maximizing this surrogate  with respect to $w$ yields a closed-form update
\begin{equation}
    w_t^+
    =
    \operatorname{clip}
    \left(
        1+\lambda c_t,\,
        w_{\min},\,
        w_{\max}
    \right).
    \label{eq:weight_update}
\end{equation}
The update implies that $\lambda$ can also be interpreted as the hyperparameter to control the weighting strength, with larger values allowing stronger adjustments to token-level weights.
Importantly, this closed-form update removes the need to parameterize or train a separate model for the weight function. Instead, we directly compute the weight $w_t^+$ for each sampled token, which substantially reduces the computational cost. 

Note that the updated $w^+$ attains a higher reward than vanilla OPD, as long as $w_t^+\not\equiv 1$.
\begin{theorem}[Reward improvement over vanilla OPD]
    \label{thm:reward_improvement}
    Under the conditions ensuring \eqref{eq:reward_lower_bound}, 
    \begin{equation}
        R\!\left(\theta^+(w^+)\right)
        -
        R\!\left(\theta_{\mathrm{vani}}^+\right)
        \geq
        \frac{\eta}{2\lambda}
        \mathbb{E}\!\left[
            \sum_t
            (w_t^+-1)^2
        \right]
        \geq 0.
        \label{eq:reward_improvement}
    \end{equation}
\end{theorem}

\section{Algorithm}
\label{sec:methodology}

In this section, we propose a practical algorithm to solve Dr. OPD.
At each round, starting from the current student $\bar\theta$, 
we sample multiple responses per prompt and obtain the teacher signals and outcome-rewards. Such outcome rewards are readily available for verifiable tasks such as math and code through answer checking and test execution \citep{guo2025deepseekr1}, and can also be provided by learned reward models for more general tasks \citep{ouyang2022training}.
Given the signals from both the teacher and the outcome reward, we estimate the token credits $c_t$ in \eqref{eq:reward_credit} and update the weights as in \eqref{eq:weight_update}.
Lastly, we update the student with one gradient step on the resulting weighted OPD.

\paragraph{Computation of credits. } Recall that $c_t = g_t^\top \nabla R(\theta_{\mathrm{vani}}^+),$ where the teacher correction $g_t$ is directly computed as in \revision{\eqref{eq:weighted_opd_gradient}} but estimating $\nabla R(\theta_{\mathrm{vani}}^+)$ can be computationally expensive. In particular, estimating $\nabla R(\theta_{\mathrm{vani}}^+)$ would require first applying a vanilla OPD update from the current $\bar\theta$ to obtain $\theta_{\mathrm{vani}}^+$, then generating a new set of responses from this updated student and using them to estimate its  gradient of reward. This would introduce an additional round of model update and response generation at every training iteration, substantially increasing
the computational cost.

To avoid the cost, we approximate $\nabla R(\theta_{\mathrm{vani}}^+)$ by  $\nabla R(\bar\theta)$ evaluated at the current student
$\bar\theta$, which can be estimated using the responses already sampled.
For a general outcome reward function, $\nabla R(\bar\theta)$ can be estimated by weighting each response's log-probability gradient by its reward.
For the verifiable math and code tasks, we can follow GRPO \citep{shao2024deepseekmathpushinglimitsmathematical} to estimate $\nabla R(\bar\theta)$ by
\begin{equation*}
    \widehat r:=\frac{1}{|\mathcal B|}\sum_{(x,y)\in\mathcal B}
    A(x,y)\sum_t\left.\nabla_\theta\log p_\theta(y_t\mid s_t)\right|_{\theta=\bar\theta},
\end{equation*}
where $A(x,y)$ denotes the group-relative advantage with the explicit expression in Appendix \ref{app:method_implementation}, and $\mathcal{B}$ denotes the batch of prompt-response pairs.
We then approximate the token credit by $ g_t^\top \widehat r$.

\paragraph{Efficient computation with Jacobian--vector product (JVP). }
For a response with $T$ tokens, we shall compute credits
$\widehat c_t=g_t^\top \widehat r$ for each token $t=1,\ldots,T$.
A naive implementation would require computing $T$  parameter gradients $g_1,\ldots,g_T \in \mathbb{R}^{|\bar\theta|},$ one for each token. Since the number of model parameters
$|\bar\theta|$ is very large, explicitly forming all $\{g_t\}_{t=1}^T$ is prohibitively expensive.

It turns out that we do not need to explicitly compute or store the parameter gradient $g_t$. Recall that
\begin{equation}
    g_t
    =
    d_t
    \left.
    \nabla_\theta \log p_\theta(y_t\mid s_t)
    \right|_{\theta=\bar\theta},
    \qquad
    d_t
    :=
    \log
    \frac{p^*(y_t\mid s_t)}
         {p_{\bar\theta}(y_t\mid s_t)}
    \in\mathbb{R}.
    \label{eq:practical_teacher_correction}
\end{equation}
Substituting this expression into
$\widehat c_t=g_t^\top\widehat r$ gives, for each token $t=1,\dots,T$,
\begin{equation}
    \widehat c_t
    =
    d_t
    \left\langle
    \left.
    \nabla_\theta \log p_\theta(y_t\mid s_t)
    \right|_{\theta=\bar\theta},
    \widehat r
    \right\rangle
    =
    d_t
    \left.
    \frac{\mathrm d}{\mathrm d\epsilon}
    \log p_{\bar\theta+\epsilon\widehat r}(y_t\mid s_t)
    \right|_{\epsilon=0},
    \label{eq:practical_jvp_credit}
\end{equation}
where the second equality follows from the directional-derivative identity. Thus, instead of explicitly constructing the $T$ parameter-dimensional gradients $\{g_t\}_{t=1}^T$, we only need the $T$ scalar directional derivatives $\left\{\left.\frac{\mathrm d}{\mathrm d\epsilon}\log p_{\bar\theta+\epsilon\widehat r}(y_t\mid s_t)\right|_{\epsilon=0}\right\}_{t=1}^T$, which can be computed altogether with a single JVP.

\paragraph{Updating weights.}
Following the closed-form update in \eqref{eq:weight_update}, we update the token weights:
\begin{equation}
    \widehat w_t
    =
    \clip\!\left(
        1
        +
        \lambda
        \frac{\widehat c_t}
             {\max(\sigma,\epsilon)},
        \,
        w_{\min},
        \,
        w_{\max}
    \right),
    \label{eq:practical_weight_update}
\end{equation}
where $\epsilon>0$ prevents division by zero, 
and $\sigma$ is the root-mean-squared (RMS) scale of the estimated credits, whose explicit expression is provided in Appendix~\ref{app:practical_details}. This RMS normalization helps stabilize the weight updates.

We then perform one gradient step on the $\{\widehat w_t\}$-weighted OPD, as in \eqref{eq:one_step_update}, but with the population expectation replaced by an empirical average. The complete procedure is in Algorithm~\ref{alg:dr_opd}.
Compared with vanilla OPD, the only additional steps are computing the gradient of the outcome rewards, the token credits, and the corresponding weights.
Since the credits can be computed with a single JVP and the weights are obtained in closed form, these computations add little overhead.

\begin{algorithm}[H]
\caption{Solver for DR-OPD}
\label{alg:dr_opd}
\begin{algorithmic}[1]
\Require Student $p_\theta$, teacher $p^*$, outcome reward function $V$, weighting strength $\lambda$, bounds $w_{\min},w_{\max}$
\State Set current student policy $\bar\theta \gets \theta$
\For{each training round}
    \State Sample multiple responses per prompt from $\bar\theta$.
    \State Evaluate responses' outcome reward and compute the estimated gradient of reward $\widehat r$.
    \State Compute the token discrepancies $d_t$ in \eqref{eq:practical_teacher_correction}.
    \State Compute credits $\widehat c_t$ by JVP using \eqref{eq:practical_jvp_credit}.
    \State Update weights $\widehat w_t$ using \eqref{eq:practical_weight_update}.
    \State Update $\bar\theta$ by gradient update of $\{\widehat w_t\}$-weighted OPD in \eqref{eq:one_step_update}.
\EndFor
\end{algorithmic}
\end{algorithm}

\section{Experimental Results}
\label{sec:experiments}

We evaluate Dr. OPD across four student--teacher configurations:
\[
(\text{Base},\text{Instruct}),\quad
(\text{Instruct},\text{Instruct}),\quad
(\text{Base},\text{RL}),\quad
(\text{Instruct},\text{RL}),
\]
where the first entry denotes the type of student model and the second denotes the type of teacher.
These terms refer to different stages of post-training.
A \emph{Base} model is the pretrained model before instruction tuning, while an
\emph{Instruct} model has been post-trained to better follow user instructions and solve downstream tasks.
An \emph{RL} model is obtained through further reinforcement learning using outcome rewards and represents a more strongly post-trained model.

These configurations form two complementary distillation regimes. In \emph{strong-to-weak} regime, a larger Instruct model serves as the teacher for either a Base or an Instruct student. In \emph{same-size} regime, we use an RL-trained teacher and a base or Instruct model of the same size as the student. 

\subsection{Experimental setup}
\label{sec:exp-setup}
\textbf{Models.}
We instantiate the four student--teacher configurations above using Qwen3 models, where models without the ``Base'' or ``RL'' suffix refer to their Instruct versions.
For strong-to-weak distillation, we use Qwen3-4B~\citep{yang2025qwen3technicalreport} as the teacher and consider two students, Qwen3-1.7B-Base and Qwen3-1.7B.
For same-size distillation, we consider Qwen3-4B-Base with an RL-trained Qwen3-4B-Base teacher, and Qwen3-4B with an RL-trained Qwen3-4B teacher \citep{yang2026learningteachergeneralizedonpolicy}.
All teachers are frozen during distillation.

\textbf{Training data and protocol.}
For mathematical reasoning, we train on DeepMath-103K~\citep{he2025deepmath103klargescalechallengingdecontaminated}. For code generation, we train on Eurus-RL-Code~\citep{cui2025processreinforcementimplicitrewards}. Following prior OPD work~\citep{yang2026learningteachergeneralizedonpolicy,sun2026reopdreliabilityadaptiverewardextrapolation}, we use a small number of post-training steps: 50 steps in all settings.
Full training details are provided in Appendix~\ref{app:hparams}.

\textbf{Evaluation.}
For math, we evaluate on AIME24, AIME25, AMC, Minerva Math~\citep{lewkowycz2022solvingquantitativereasoningproblems}, and OlympiadBench~\citep{he2024olympiadbenchchallengingbenchmarkpromoting}, and report accuracy averaged over 16 sampled responses per problem (avg@16).
For code, we evaluate on HumanEval+ and MBPP+ from EvalPlus~\citep{liu2023codegeneratedchatgptreally}, together with LiveCodeBench~\citep{jain2024livecodebenchholisticcontaminationfree}, and report avg@4.
We evaluate checkpoints every 10 optimization steps on the evaluated datasets and report the best checkpoint for each method, as adopted in works ~\citep{ding2026doesonpolicydistillationreally,zhao2026selfdistilledreasoneronpolicyselfdistillation}. 

\textbf{Baselines.}
We compare Dr. OPD with four OPD-based baselines.
\textit{OPD}~\citep{gu2026minillmonpolicydistillationlarge} is the standard on-policy distillation baseline with uniform token weighting.
\textit{ExOPD}~\citep{yang2026learningteachergeneralizedonpolicy} extrapolates the teacher's distillation signal with the goal of moving the student beyond the teacher.
With outcome reward signals, 
\textit{GRPD}~\citep{lin2026onpolicydistillationverifiablereward} considers reward-gated distillation, while \textit{OPD+GRPO}~\citep{coreteam2026mimov2flashtechnicalreport} adds an additional GRPO objective alongside the OPD objective.
Implementation details are provided in Appendix~\ref{app:hparams}.

\subsection{Distillation results}
\label{sec:exp-results}

\paragraph{Strong-to-weak distillation.}
We distill the Qwen3-4B teacher into the Qwen3-1.7B and Qwen3-1.7B-Base students. As shown in Table~\ref{tab:s2w}, Dr. OPD achieves the best average performance on both math and code for both students. Notably, vanilla OPD still leaves a substantial gap between the student and the stronger teacher, suggesting that simply mimicking a stronger, larger teacher is insufficient to bring the student close to the teacher's performance.

\textbf{A smaller student can outperform its larger teacher.}
For the 1.7B student, Dr. OPD even surpasses the larger 4B teacher in average math performance (35.0 vs.\ 34.1). This improvement reflects the design of Dr. OPD, which does not merely mimic the teacher but also uses the student's outcome reward to determine which teacher signals to follow more closely. However, simply adding outcome rewards is not sufficient, since both OPD+GRPO and GRPD also use such rewards but yield only marginal improvements. This suggests that how the outcome reward is incorporated into distillation is important.

For the 1.7B-Base student, Dr. OPD still improves over vanilla OPD by 1.8 points on math, while the gain is much larger on code, reaching 6.4 points. Compared with the 1.7B student, the 1.7B-Base student shows a smaller gain on math. We attribute this mainly to its much lower starting accuracy. Since the Base student produces fewer successful rollouts, the outcome-reward signal used to estimate token credits is less informative for identifying useful teacher supervision. Despite this, Dr. OPD still achieves the best average performance among all compared methods.

\begin{table}[!ht]
\centering
\footnotesize
\setlength{\tabcolsep}{3.4pt}
\renewcommand{\arraystretch}{1.08}

\begin{adjustbox}{max width=0.95\linewidth}
\begin{tabular}{@{}lcccccc@{\hspace{0.9em}}cccc@{}}
\toprule
& \multicolumn{6}{c}{\textbf{Math}}
& \multicolumn{4}{c}{\textbf{Code}} \\
\cmidrule(lr){2-7}
\cmidrule(lr){8-11}
Method
& AIME24 & AIME25 & AMC & Minerva & Olym. & \textbf{Avg.}
& HE+ & MBPP+ & LCB & \textbf{Avg.} \\
\midrule

Teacher (Qwen3-4B)
& 22.5 & 16.9 & 60.5 & 27.1 & 43.7 & 34.1
& 77.1 & 64.8 & 16.1 & 52.7 \\

\midrule
\multicolumn{11}{@{}l}{\textit{Student: Qwen3-1.7B}} \\
\addlinespace[1pt]

Initial
& 12.3 & 9.0 & 39.7 & 19.1 & 32.9 & 22.6
& 59.8 & 52.5 & 11.9 & 41.4 \\

\noalign{\vskip 1.5pt}
\cdashline{1-11}
\noalign{\vskip 1.5pt}

OPD
& 15.4 & 12.5 & 43.7 & 19.4 & 35.8 & 25.3
& 60.4 & 54.9 & 15.0 & 43.4 \\

ExOPD
& 17.5 & 10.0 & 44.0 & 20.2 & 36.6 & 25.7
& 61.1 & 56.5 & 15.4 & 44.3 \\

OPD+GRPO
& 16.0 & 13.1 & 42.9 & 21.0 & 35.1 & 25.6
& 60.7 & 55.4 & 14.9 & 43.7 \\

GRPD
& 15.8 & 9.2 & 43.5 & 19.4 & 35.4 & 24.7
& 64.0 & 52.0 & 14.9 & 43.6 \\

\rowcolor{dropdrow}
\textbf{Dr. OPD}
& \textbf{25.6}\gain{10.2}
& \textbf{23.5}\gain{11.0}
& \textbf{56.9}\gain{13.2}
& \textbf{24.6}\gain{5.2}
& \textbf{44.4}\gain{8.6}
& \textbf{35.0}\gain{9.7}
& \textbf{64.2}\gain{3.8}
& \textbf{60.7}\gain{5.8}
& \textbf{20.4}\gain{5.4}
& \textbf{48.4}\gain{5.0} \\

\midrule
\multicolumn{11}{@{}l}{\textit{Student: Qwen3-1.7B-Base}} \\
\addlinespace[1pt]

Initial
& 1.7 & 1.5 & 10.8 & 4.4 & 7.8 & 5.2
& 5.5 & 4.2 & 1.3 & 3.7 \\

\noalign{\vskip 1.5pt}
\cdashline{1-11}
\noalign{\vskip 1.5pt}

OPD
& 6.9 & 5.0 & 28.8 & 20.1 & 23.3 & 16.8
& 54.0 & 55.0 & 11.7 & 40.2 \\

ExOPD
& 6.7 & \textbf{5.6} & 27.5 & 20.6 & 23.8 & 16.8
& 61.1 & 54.7 & 14.3 & 43.4 \\

OPD+GRPO
& 8.1 & 5.2 & 29.3 & 21.0 & 23.7 & 17.5
& 63.1 & 55.1 & 15.0 & 44.4 \\

GRPD
& 6.9 & \textbf{5.6} & 28.8 & 20.0 & 23.0 & 16.9
& 59.3 & 50.3 & 13.0 & 40.9 \\

\rowcolor{dropdrow}
\textbf{Dr. OPD}
& \textbf{9.2}\gain{2.3}
& 4.8\loss{0.2}
& \textbf{31.3}\gain{2.5}
& \textbf{21.3}\gain{1.2}
& \textbf{26.5}\gain{3.2}
& \textbf{18.6}\gain{1.8}
& \textbf{64.3}\gain{10.3}
& \textbf{58.5}\gain{3.5}
& \textbf{16.9}\gain{5.2}
& \textbf{46.6}\gain{6.4} \\

\bottomrule
\end{tabular}
\end{adjustbox}

\caption{Strong-to-weak distillation.
We distill Qwen3-4B into Instruct and Base Qwen3-1.7B students on
math and code generation.
Colored subscripts report the absolute change over vanilla OPD.
Bold denotes the best
distillation result within each student setting.
}
\label{tab:s2w}
\end{table}

\paragraph{Same-size distillation.}
We next consider the setting where the student and teacher have the same model size, with the teacher obtained through further RL training. As shown in Table~\ref{tab:samesize}, vanilla OPD already brings the student close to or even beyond the teacher, leaving limited room for further gains by simply mimicking the teacher. Despite this, Dr. OPD still achieves the best average performance for both setups, improving over vanilla OPD by 2.3 and 3.0 points, respectively.

\textbf{Surpassing the RL-trained teacher.}
Notably, Dr. OPD surpasses the RL-trained teacher in both settings. For the Instruct student, it exceeds the teacher on all five benchmarks, a level of consistent improvement not achieved by any competing method. For the Base student, Dr. OPD exceeds the teacher on four of the five benchmarks and improves the average performance from 34.1 to 36.0.

\begin{table}[!ht]
\centering
\footnotesize
\setlength{\tabcolsep}{5.2pt}
\renewcommand{\arraystretch}{1.08}

\begin{adjustbox}{width=0.70\linewidth}
\begin{tabular}{@{}lcccccc@{}}
\toprule
Method
& AIME24 & AIME25 & AMC & Minerva & Olym. & \textbf{Avg.} \\
\midrule

\multicolumn{7}{@{}l}{
\textit{Student: Qwen3-4B \qquad Teacher: Qwen3-4B-RL}
} \\
\addlinespace[1pt]

Teacher
& 54.0 & 46.5 & 81.0 & 36.9 & 57.7 & 55.2 \\

Initial
& 22.5 & 16.9 & 60.5 & 27.1 & 43.7 & 34.1 \\

\noalign{\vskip 1.5pt}
\cdashline{1-7}
\noalign{\vskip 1.5pt}

OPD
& 55.8 & 45.8 & 81.0 & 37.3 & 57.9 & 55.6 \\

ExOPD
& 56.9 & 43.3 & 82.3 & 37.6 & 56.6 & 55.3 \\

OPD+GRPO
& 55.6 & 46.3 & 82.0 & 37.5 & 57.7 & 55.8 \\

GRPD
& 55.8 & 49.2 & 81.6 & 37.6 & 57.0 & 56.2 \\

\rowcolor{dropdrow}
\textbf{Dr. OPD}
& \textbf{58.3}\gain{2.5}
& \textbf{51.7}\gain{5.8}
& \textbf{82.9}\gain{1.9}
& \textbf{37.8}\gain{0.5}
& \textbf{58.5}\gain{0.6}
& \textbf{57.8}\gain{2.3} \\

\midrule

\multicolumn{7}{@{}l}{
\textit{Student: Qwen3-4B-Base \qquad Teacher: Qwen3-4B-Base-RL}
} \\
\addlinespace[1pt]

Teacher
& 23.1 & 22.1 & 55.4 & 26.4 & 43.5 & 34.1 \\

Initial
& 11.9 & 8.5 & 36.8 & 9.7 & 27.6 & 18.9 \\

\noalign{\vskip 1.5pt}
\cdashline{1-7}
\noalign{\vskip 1.5pt}

OPD
& 21.0 & 19.2 & 56.0 & 27.4 & 41.5 & 33.0 \\

ExOPD
& \textbf{26.9} & 19.2 & 56.7 & 25.6 & 43.1 & 34.3 \\

OPD+GRPO
& 21.9 & \textbf{21.0} & 57.5 & 27.2 & 42.3 & 34.0 \\

GRPD
& 21.5 & 18.8 & 56.8 & 27.6 & 43.1 & 33.5 \\

\rowcolor{dropdrow}
\textbf{Dr. OPD}
& 25.4\gain{4.4}
& 20.0\gain{0.8}
& \textbf{58.4}\gain{2.4}
& \textbf{31.9}\gain{4.5}
& \textbf{44.1}\gain{2.6}
& \textbf{36.0}\gain{3.0} \\

\bottomrule
\end{tabular}
\end{adjustbox}

\caption{Same-size distillation.
We evaluate Instruct and Base students distilled from their corresponding
RL-trained teachers on mathematical reasoning.
Colored subscripts in the Dr. OPD rows report the absolute improvement
over vanilla OPD.
Teacher and Initial are reference points; bold denotes the best
distillation result within each student--teacher setting.
}
\label{tab:samesize}

\end{table}

\subsection{Ablation Study: Effect of \texorpdfstring{$\lambda$}{lambda}}
\label{sec:exp-ablation}

The weighting strength $\lambda$ is the main hyperparameter in Algorithm \ref{alg:dr_opd},
controlling how strongly the token weights can deviate from the vanilla OPD value of $1$. When $\lambda=0$, Dr. OPD reduces to vanilla OPD. As shown in
Table~\ref{tab:ablation-lambda}, every positive value of $\lambda$ improves
over OPD on all five benchmarks and in average performance, suggesting that
Dr. OPD does not require precise tuning of $\lambda$. Among the values tested,
$\lambda=0.4$ performs best overall and is used in our main experiments.

\begin{table}[!ht]
\centering
\footnotesize
\setlength{\tabcolsep}{5.0pt}
\renewcommand{\arraystretch}{1.08}

\begin{adjustbox}{width=0.6\linewidth}
\begin{tabular}{@{}lcccccc@{}}
\toprule
$\lambda$
& AIME24 & AIME25 & AMC & Minerva & Olym. & \textbf{Avg.} \\
\midrule

0 (OPD)
& 15.4 & 12.5 & 43.7 & 19.4 & 35.8 & 25.3 \\

\noalign{\vskip 1.5pt}
\cdashline{1-7}
\noalign{\vskip 1.5pt}

0.2
& 18.5 & 13.5 & 47.8 & 22.8 & 37.7 & 28.1 \\

0.3
& 19.4 & 14.0 & 45.0 & 23.2 & 37.5 & 27.8 \\

\rowcolor{dropdrow}
\textbf{0.4}
& \textbf{25.6}
& \textbf{23.5}
& \textbf{56.9}
& 24.6
& \textbf{44.4}
& \textbf{35.0} \\

0.5 & 24.0 & 19.2 & 54.7 & \textbf{25.0} & 42.6 & 33.1 \\
0.6 & 23.8 & 19.8 & 54.1 & 23.3 & 44.0 & 33.0 \\

\bottomrule
\end{tabular}
\end{adjustbox}

\caption{Effect of the weighting strength $\lambda$ on mathematical reasoning.
We vary $\lambda$ for the Qwen3-1.7B student distilled from Qwen3-4B.
$\lambda=0$ reduces to vanilla OPD, while larger $\lambda$
allows stronger token-level weight adjustments. Bold denotes the best result in each column.
}
\label{tab:ablation-lambda}
\end{table}

\section{Related Works}
\label{sec:related_work}

\textbf{On-policy distillation.}
Learning from teacher feedback on student-generated sequences addresses the mismatch between training trajectories and the contexts encountered by the student at inference. 
GKD studies this principle with a family of divergence objectives \citep{agarwal2024onpolicydistillationlanguagemodels}, while MiniLLM develops reverse-KL optimization for language-model distillation \citep{gu2026minillmonpolicydistillationlarge}. Subsequent work demonstrates its effectiveness for reasoning post-training \citep{yang2025qwen3technicalreport}. Multi-teacher OPD also consolidates specialist capabilities in MiMo-V2-Flash, DeepSeek-V4, and Kimi K3 \citep{coreteam2026mimov2flashtechnicalreport,deepseekai2026deepseekv4,kimiteam2026kimik3}. Beyond direct teacher matching, G-OPD and ExOPD introduce a reference policy and reward extrapolation \citep{yang2026learningteachergeneralizedonpolicy}, and OPD$^2$ uses the difference between a reasoning-tuned teacher and its base model as the distillation signal \citep{heo2026opd2}. Analyses further identify teacher--student compatibility as a factor in successful transfer \citep{li2026rethinking}. 

\paragraph{Trust and weighting in OPD.}
Several methods already adapt the strength or location of teacher supervision. REOPOLD stabilizes learning through reward clipping and entropy-based token sampling \citep{ko2026reopold}. TrOPD identifies trustworthy regions through teacher--student agreement and treats outlier tokens with alternative objectives \citep{xing2026tropd}. IW-OPD weights tokens using accumulated distribution discrepancy to address deteriorating supervision along long student trajectories \citep{xie2026positionbias}. REOPD combines token-level compatibility weights with a batch-level budget to adapt extrapolation beyond the teacher \citep{sun2026reopdreliabilityadaptiverewardextrapolation}. For multiple teachers, TrustMOPD allocates token-level supervision using each specialist's displacement from a shared pre-RL reference \citep{sun2026trustmopd}. Outcome feedback also supports selective distillation: OPDVR gates sampled-token signals according to response correctness \citep{lin2026onpolicydistillationverifiablereward}, while VG-OPD verifies an expert's advantage on individual answer criteria and localizes its supervision through token disagreement \citep{xu2026vgopd}. 

Concurrent work SparseOPD \citep{liu2026extremely} studies extremely sparse supervision for OPD, showing that retaining only one or two token-level teacher signals per trajectory can match or outperform dense vanilla OPD. SparseOPD can be viewed as a binary-weighted OPD  that masks most token-level signals, while Dr. OPD, instead, adaptively assigns continuous weights according to how each teacher signal affects the student's performance. Thus, SparseOPD studies how little supervision is sufficient, while Dr. OPD studies how teacher supervision should be weighted to maximize student performance.
Closely related to our motivation, Unmasking OPD diagnoses teacher signals through alignment with an estimated success-improving gradient in token-logit space \citep{armandpour2026unmasking}. Dr. OPD uses the sensitivity of the updated student's reward to each supervision weight as an optimization criterion. This yields parameter-space credits for teacher corrections and a clipped weighting rule derived from a reward lower bound relative to vanilla OPD.

\paragraph{Bilevel optimization, reweighting, and influence.}
Selecting supervision through the performance of the learner connects to bilevel optimization and differentiation through learning dynamics \citep{franceschi2018bilevel,shaban2019truncated}. A common strategy in this literature is to evaluate the outer objective through one or a few gradient updates of the inner learner, as in MAML \citep{finn2017model}. This idea is also used for data reweighting, where \citet{ren2019learningreweightexamplesrobust} selects example weights through a one-step learner update evaluated on clean validation data, leading to a gradient-alignment criterion. Influence-function methods study how upweighting training examples affects the learned model, typically through an inverse-Hessian correction \citep{koh2017influence}, while TracIn uses gradient inner products along optimization trajectories \citep{pruthi2020tracin}. Our method brings these ideas to token-level teacher supervision: we measure the reward sensitivity of a one-step OPD update, use outcome reward as the outer objective, and derive finite weight updates from a lower bound on reward improvement.

\section{Conclusion}

We proposed \textbf{Dr. OPD}, a bilevel optimization formulation that finds the optimal weighted OPD objective for improving students' task performance. The lower-level problem trains the student under a given weight function, while the upper-level problem selects the weights according to the reward of the resulting student. To solve Dr. OPD efficiently, we developed an iterative solver with closed-form weight updates followed by a standard gradient update on the weighted OPD. Across all evaluated experimental setups, Dr. OPD consistently outperforms existing baselines and can even enable students to surpass their teachers.

\subsection*{AI Use Statement}
In this work, we used generative AI tools as follows. Among tasks requiring disclosure, AI tools assisted with implementing and debugging code, preparing scripts for experiments, and checking and refining mathematical derivations and proofs. All AI-assisted code was inspected and executed by the authors, all reported experimental results were obtained from our own training and evaluation runs, and all AI-assisted mathematical arguments were independently verified by the authors.

We did not use generative AI tools for research ideation, developing the proposed method, designing the experiments, selecting hyperparameters, or providing methodological feedback. Generative AI tools did not autonomously run experiments or generate experimental results.

Among the tasks for which disclosure is recommended, we used generative AI tools to assist with producing scientific figures and plotting code, as well as drafting and editing the manuscript for clarity and readability. All AI-assisted text were reviewed and revised by the authors. We take full responsibility for the final content of this work, including all text, claims, results, and artifacts produced with the aid of generative AI.

\bibliography{ref_new}
\bibliographystyle{plainnat}

\clearpage
\appendix
\numberwithin{equation}{section}

\section{Hyperparameters and Experimental Setup}
\label{app:hparams}

\subsection{Shared training configuration}
All methods are implemented on top of \texttt{verl}~\citep{Sheng_2025} and share
the same trainer, rollout engine, and outcome verifier; they differ only in how the
per-token learning signal is constructed. Each run uses a single node with 8$\times$H100 GPUs.

Both RL teachers are public checkpoints that we keep frozen, without any further
training: the math-RL Qwen3-4B of \citet{yang2026learningteachergeneralizedonpolicy}
and the GRPO-trained Qwen3-4B-Base of \citet{li2026rethinking}. Within each student--teacher setting, all methods use the same frozen
teacher, so their comparison isolates the choice of training objective.

Table~\ref{tab:app-shared} lists the hyperparameters shared by all methods and
both tasks. The learning rate,
batch sizes, rollout settings, and sequence limits are identical for every
method. Table~\ref{tab:app-methods} lists the method-specific settings.
Dr.\ OPD reweights all valid response tokens, using a separate credit RMS
for each prompt and the bounds $w_{\min}=0.001$ and $w_{\max}=3$;
Appendix~\ref{app:method_implementation} gives the full implementation.

\begin{table}[!ht]
\centering
\caption{Hyperparameters shared by all methods and both tasks.}
\label{tab:app-shared}
\footnotesize
\setlength{\tabcolsep}{6pt}
\begin{tabular}{ll}
\toprule
Hyperparameter & Value \\
\midrule
Learning rate & $1\times10^{-5}$ \\
LR schedule & constant \\
Gradient clipping & 1.0 \\
Batch / Mini-batch size & 128 \\
Rollouts per prompt $K$ & 8 \\
Rollouts per step & 1024 \\
Max prompt length & 1024 tokens \\
Max response length & 12288 tokens \\
Temperature & 1.0 \\
\bottomrule
\end{tabular}
\end{table}

\begin{table}[!ht]
\centering
\caption{Method-specific hyperparameters. All other settings follow
Table~\ref{tab:app-shared}.}
\label{tab:app-methods}
\footnotesize
\begin{tabular}{ll}
\toprule
Method & Hyperparameters \\
\midrule
OPD & None \\
ExOPD & $\lambda = 1.25$; the initial student as reference policy \\
OPD+GRPO & Equal weighting (1:1) of the distillation and outcome terms \\
GRPD & None (default settings of the released method) \\
Dr.\ OPD & $\lambda = 0.4$; $w_{\min} = 0.001$; $w_{\max}=3$  \\
\bottomrule
\end{tabular}
\end{table}

\FloatBarrier

\subsection{Training Data}
\label{app:training_data}
For mathematical reasoning, we train on DeepMath-103K~\citep{he2025deepmath103klargescalechallengingdecontaminated}.
The dataset provides difficulty labels. We train the Qwen3-1.7B-Base student on
the $[3,4)$ band and every other student on the $[6,10]$ band.
For code generation, we train on Eurus-RL-Code~\citep{cui2025processreinforcementimplicitrewards},
again with the Qwen3-1.7B-Base student as the exception: it is trained on an
easier corpus drawn from APPS (introductory)~\citep{hendrycks2021measuringcodingchallengecompetence} and
TACO (EASY)~\citep{li2023tacotopicsalgorithmiccode}. These easier subsets increase the frequency of groups containing both
correct and incorrect responses for the Base student. Such groups provide
nonzero group-relative advantages; groups with identical rewards have
zero advantages.

\FloatBarrier

\subsection{Prompt Templates and Stopping Rules}
\label{app:prompts}

\paragraph{Chat template}
For all students except Qwen3-1.7B-Base, we use the model's own chat
template with thinking disabled (\texttt{enable\_thinking=False}).
The problem statement forms a single user message, followed by the
suffix supplied with the training set. The math template is shown below:

\begin{promptbox}[title={Zero-shot template, Qwen3-1.7B student (training and evaluation)}]
{question}
Please reason step by step, and put your final answer within \boxed{}.
\end{promptbox}

\paragraph{Few-shot template}
Qwen3-1.7B-Base has no chat template and does not reliably follow the
zero-shot instruction format. For this student, we use the fixed four-shot
math prompt below.

\begin{promptbox}[title={4-shot template, Qwen3-1.7B-Base student (training and evaluation)}]
Solve the following math problems. Reason step by step, and put your final
answer within \boxed{}.

Problem:
What is $1+2+3+\cdots+10$?

Solution:
The sum of the first $n$ positive integers is $\frac{n(n+1)}{2}$. For $n=10$
this is $\frac{10\cdot 11}{2}=55$. The final answer is $\boxed{55}$.

Problem:
Solve for $x$: $2x+3=11$.

Solution:
Subtracting $3$ from both sides gives $2x=8$. Dividing both sides by $2$ gives
$x=4$. The final answer is $\boxed{4}$.

Problem:
A rectangle has length $8$ and width $5$. What is its area?

Solution:
The area of a rectangle is length times width, so it is $8\times 5=40$. The
final answer is $\boxed{40}$.

Problem:
A fair coin is flipped twice. What is the probability that both flips are heads?

Solution:
The flips are independent, each heads with probability $\frac{1}{2}$. So both
are heads with probability $\frac{1}{2}\cdot\frac{1}{2}=\frac{1}{4}$. The final
answer is $\boxed{\frac{1}{4}}$.

Problem:
{question}

Solution:
\end{promptbox}

\paragraph{Stopping rule.}
With the few-shot prompt, the student may continue by generating additional
problems. We therefore truncate each response at the first completed \verb|\boxed{...}| or at the next-problem delimiter \verb|"\nProblem:"|,
whichever comes first; everywhere else generation ends at the end-of-sequence token or at the $12{,}288$-token limit. Truncation is applied to the sampled token IDs after generation. We use
the same stopping rule for every method during both training and evaluation.

\FloatBarrier

\section{Implementation of Dr. OPD}
\label{app:method_implementation}

For simplicity, Algorithm~\ref{alg:dr_opd} uses the raw GRPO estimate of the reward gradient $\widehat r$ to compute the token credits $c_t$. In our implementation, we additionally smooth $\widehat r$ and apply the coordinate-wise scaling induced by the actor optimizer, matching the transformation applied to the reward gradient in standard GRPO policy optimization.

\subsection{Group-relative reward signal}

Index training rounds by $k$, prompts by $b$, and the $K$ responses to a
prompt by $i$. We suppress $k$ on rollout-level quantities when the round
is clear. Let $r_{bi}=V(x_b,y_{bi})$, and let $T_{bi}$ be the number of
valid response tokens. The group-relative advantages are
\begin{equation*}
 \overline r_b=\frac1K\sum_{i=1}^K r_{bi},
 \qquad
 s_b=\left[\frac{1}{K-1}\sum_{i=1}^K
                    (r_{bi}-\overline r_b)^2\right]^{1/2},
 \qquad
 A_{bi}=\frac{r_{bi}-\overline r_b}{s_b+\epsilon_A}.
\end{equation*}
Here $K=8$ and $\epsilon_A>0$ stabilizes the denominator. If all rewards
in a group agree, its advantages are zero. With advantages and response
lengths held fixed, the length-normalized reward direction is
\begin{equation*}
 G_k=\sum_b\frac1K\sum_{i=1}^K\frac{A_{bi}}{T_{bi}}
        \sum_{t=1}^{T_{bi}}
        \left.\nabla_\theta\log p_\theta(y_{bi,t}\mid s_{bi,t})
        \right|_{\theta=\bar\theta_k},
\end{equation*}
up to the common distributed averaging factor. It is computed by
differentiating a negative log-probability surrogate with detached
coefficients $A_{bi}/(KT_{bi})$ and reversing the gradient's sign.
The factor $1/T_{bi}$ normalizes each response's contribution; a response
with no valid tokens contributes zero.

\subsection{Optimizer-scaled reward direction}
\label{app:reward_direction}

We maintain a reward first moment, separate from the actor optimizer's
first moment:
\begin{equation*}
 m_k^R=\beta_1m_{k-1}^R+(1-\beta_1)G_k,
 \qquad m_0^R=0.
\end{equation*}
The direction used in round $k$ is constructed from the state available
before the current actor update. For $k>1$, let
$\widehat m_{k-1}^R=m_{k-1}^R/(1-\beta_1^{k-1})$. Let
$v^{\mathrm{actor}}$ be the actor AdamW optimizer's stored second moment,
$j$ its number of completed optimizer steps, and
$\widehat v^{\mathrm{actor}}=v^{\mathrm{actor}}/(1-\beta_2^j)$.
The credit direction is
\begin{equation*}
 u_k=\frac{\widehat m_{k-1}^R}
           {\sqrt{\widehat v^{\mathrm{actor}}}+\epsilon_{\mathrm{Adam}}},
\end{equation*}
where the operations are elementwise. Thus the numerator uses reward
momentum, while the denominator reuses the actor's second moment.
The current $G_k$ updates the reward momentum for subsequent rounds.
If the required optimizer or momentum state is unavailable, all weights
are set to one. The decoupled weight-decay term is excluded from $u_k$.

\subsection{Token credits, RMS scale, and weight clipping}
\label{app:practical_details}

For each sampled token, define the teacher log-ratio and credit by
\begin{align*}
 d_{bi,t}
 &=\log\frac{p^*(y_{bi,t}\mid s_{bi,t})}
                 {p_{\bar\theta_k}(y_{bi,t}\mid s_{bi,t})},\\
 \widehat c_{bi,t}
 &=d_{bi,t}\left.
       \frac{\mathrm d}{\mathrm d\alpha}
       \log p_{\bar\theta_k+\alpha u_k}(y_{bi,t}\mid s_{bi,t})
       \right|_{\alpha=0}
 =g_{bi,t}^{\top}u_k.
\end{align*}
A forward-mode JVP computes these directional derivatives without
materializing a parameter gradient for each token. The student replica
used for the JVP is synchronized with the actor, and dropout is disabled.
The derivative is evaluated by automatic differentiation, without a
finite-difference perturbation.

Let $\mathcal I_{kb}$ contain all valid response-token positions across
the $K$ rollouts for prompt $x_b$ in round $k$:
\begin{equation*}
 \mathcal I_{kb}=\{(i,t):1\le i\le K,\ 1\le t\le T_{bi}\},
 \qquad N_{kb}=|\mathcal I_{kb}|=\sum_{i=1}^K T_{bi}.
\end{equation*}
The explicit RMS scale used in \eqref{eq:practical_weight_update} is
\begin{equation*}
 \boxed{
 \sigma_{kb}
 =\left(\frac{1}{N_{kb}}
          \sum_{i=1}^{K}\sum_{t=1}^{T_{bi}}
             \widehat c_{bi,t}^{\,2}\right)^{1/2}
 }
 \qquad (N_{kb}>0).
\end{equation*}
This is a token-weighted RMS across the responses to one prompt, without
subtracting the mean credit. Prompt and padding tokens are excluded;
valid response tokens with zero credit are included. No probability-based token selection is applied.
For each $(i,t)\in\mathcal I_{kb}$, we set
\begin{equation*}
 \widehat w_{bi,t}
 =\clip\!\left(
      1+\lambda\frac{\widehat c_{bi,t}}
                         {\max(\sigma_{kb},\epsilon_\sigma)},
      \,0.001,\,3\right),
 \qquad \lambda=0.4
\end{equation*}
in the main experiments; the ablation varies $\lambda$.
Here $\epsilon_\sigma>0$ is the stability constant denoted by
$\epsilon$ in \eqref{eq:practical_weight_update}. If all credits for a
prompt are zero, its valid tokens receive weight one. If $N_{kb}=0$,
we set $\sigma_{kb}=0$ and that prompt contributes no response tokens
to the actor loss. The positive lower bound preserves strictly positive
weights, as required by the weighted-distillation formulation.

\subsection{Actor update}

The detached products $\widehat w_{bi,t}d_{bi,t}$ are passed as token
advantages to the actor loss used by vanilla OPD. The actor uses a
PPO-style probability-ratio objective, including ratio clipping and
dual clipping for negative advantages, token-mean aggregation, and
AdamW. At the reference policy, where the probability ratio equals one,
the ascent gradient is the token average of
$\widehat w_{bi,t}g_{bi,t}$. Gradients used to construct the reward
direction are cleared before the actor update. The weights and teacher
signals remain detached during that update.

\FloatBarrier

\section{Proofs and Theoretical Details}
\label{app:proofs}

This section collects the assumptions, derivations, and proofs used in
Sections~\ref{sec:formulation}--\ref{sec:local_weighting}. We first specify
the reference distribution for one optimization round, then establish the
weighted-distillation properties and the reward-improvement guarantee.

\subsection{Notation and regularity conditions}
\label{app:local_analysis}
\label{app:reward_bound}

Fix the current student $\bar\theta$. We assume a finite vocabulary and
strictly positive student and teacher probabilities. A feasible weight
function is a measurable function in
\[
 \mathcal W=\{w:w_{\min}\le w(s,v)\le w_{\max}\},
 \qquad 0<w_{\min}<1<w_{\max}<\infty.
\]
The reference context distribution and $w$ are held fixed when
we differentiate the distillation objective. In particular, the loss for
this round is
\begin{equation}
 D_w^{\bar\theta}(\theta)
 :=\EE_{x,\,y\sim p_{\bar\theta}(\cdot\mid x)}
 \left[\sum_t \KL_{w(s_t,\cdot)}
 \bigl(p_\theta(\cdot\mid s_t)\|p^*(\cdot\mid s_t)\bigr)\right].
 \label{eq:frozen_objective}
\end{equation}
We suppress the superscript $\bar\theta$ below and in the main text.
The reference distribution is refreshed between rounds.

Define the reference measure on context--token pairs by
\begin{equation*}
 \int f(s,v)\,\mathrm d\mu(s,v)
 :=\EE_{x,\,y\sim p_{\bar\theta}(\cdot\mid x)}
       \left[\sum_{t=1}^T f(s_t,y_t)\right]
\end{equation*}
for nonnegative measurable $f$. Its total mass is $\EE[T]$; thus $\mu$
is an unnormalized token-occupancy measure. It includes both context
visitation and the conditional probability of the sampled token. We use
$\langle f,h\rangle_\mu=\int fh\,\mathrm d\mu$ and
$\|f\|_\mu^2=\int f^2\,\mathrm d\mu$ for scalar functions, and write
$g(s_t,y_t)=g_t$ for the teacher correction in
\eqref{eq:weighted_opd_gradient}.

\begin{assumption}[Local regularity and integrability]
\label{ass:local_regularity}
Let $\mu$ denote the fixed reference distribution over sampled trajectories, and let
$\mathcal W=\{w:w_{\min}\leq w\leq w_{\max}\}$.
Assume the following conditions hold.

\begin{enumerate}[label=(\roman*), leftmargin=2em]
    \item \textbf{Integrability.}
    The expected response length is finite,
    \[
        \EE_{\mu}[T] < \infty.
    \]
    Moreover, the weighted OPD objective is finite in a neighborhood of
    $\bar\theta$, and differentiation can be interchanged with expectation.
    In particular,
    \[
        \nabla_\theta D_w(\theta)
        =
        \EE_{\mu}\!\left[
            \nabla_\theta \ell_w(\theta;Z)
        \right],
        \qquad
        w\in\mathcal W,
    \]
    for all $\theta$ in this neighborhood.

    \item \textbf{Finite second moment.}
    The token-level teacher correction satisfies
    \begin{equation*}
        M
        :=
        \EE_{\mu}\!\left[
            \sum_{t=1}^{T}\|g_t\|^2
        \right]
        <\infty.
    \end{equation*}

    \item \textbf{Local smoothness of the reward.}
    There exists an open convex set $U$ such that
    \[
        \{\bar\theta\}
        \cup
        \{\theta^+(w):w\in\mathcal W\}
        \subset U,
    \]
    and $R$ is continuously differentiable on $U$ with
    $L_R$-Lipschitz gradient:
    \[
        \|\nabla R(\theta)-\nabla R(\theta')\|
        \leq
        L_R\|\theta-\theta'\|,
        \qquad
        \forall\,\theta,\theta'\in U.
    \]
\end{enumerate}
All parameter norms are Euclidean.
\end{assumption}

Cauchy--Schwarz gives $\int\|g\|\,\mathrm d\mu\le\sqrt{\EE[T]M}$,
so the response in \eqref{eq:one_step_update} satisfies
\begin{equation*}
 \|\theta^+(w)-\bar\theta\|
 \le\eta w_{\max}\sqrt{\EE[T]M}.
\end{equation*}
Thus a sufficiently small student step size keeps all responses in a
fixed neighborhood where the required smoothness holds. The condition
$\lambda\eta L_RM\le1$ is equivalent to
$\lambda\le1/(\eta L_RM)$ when $L_RM>0$; when $L_RM=0$, it imposes
no further restriction on positive $\lambda$.

Unlike the fixed-reference distillation loss, the reward $R(\theta)$
is evaluated under the student's own response distribution. Under the
corresponding differentiation and summability conditions,
\begin{equation*}
 \nabla R(\theta)
 =\EE_{x,\,y\sim p_\theta}
       \left[V(x,y)\sum_t\nabla_\theta
                     \log p_\theta(y_t\mid s_t)\right].
\end{equation*}
This follows by differentiating
$p_\theta(y\mid x)=\prod_t p_\theta(y_t\mid s_t)$, with termination
included in the response convention. 

\subsection{Properties of weighted OPD}
\label{app:weighted_objective}

\subsubsection{Proof of Theorem~\ref{thm:weighted_opd_validity}}
\label{app:weighted_validity}

\begin{proof}
For positive $a,b$, define
\begin{equation*}
 d(a,b)=a\log(a/b)-a+b=b\phi(a/b),
 \qquad \phi(u)=u\log u-u+1.
\end{equation*}
Since $\phi'(u)=\log u$ and $\phi''(u)=1/u>0$, its unique minimum
is $\phi(1)=0$. Hence $d(a,b)\ge0$, with equality if and only if
$a=b$. At each context $s$,
\[
 \KL_{w(s,\cdot)}(p_\theta\|p^*)
 =\sum_v w(s,v)d\bigl(p_\theta(v\mid s),p^*(v\mid s)\bigr)\ge0.
\]
Because every weight is strictly positive, equality holds precisely
when the student and teacher probabilities agree for every token.
Taking the reference expectation proves $D_w(\theta)\ge0$.
The expected loss is zero if and only if the distributions agree at
almost every context under the reference context measure. In
particular, the claim holds at every context of positive reference
mass; the objective places no restriction on contexts of zero measure.
Finally, when $w\equiv1$, the correction terms sum to zero by
normalization, leaving the ordinary reverse KL and the vanilla OPD loss.
\end{proof}

The linear correction is needed for token-dependent weights. For
example, if $p=(1/4,3/4)$, $p^*=(1/2,1/2)$, and $w=(4,1)$, then
$\sum_v w_vp_v\log(p_v/p_v^*)=-\log2+(3/4)\log(3/2)<0$.
The corrected divergence remains nonnegative term by term. When
$w(s,v)=a(s)$ is independent of the token, the correction cancels and
we recover $a(s)\KL(p_\theta(\cdot\mid s)\|p^*(\cdot\mid s))$.

\subsubsection{Derivation of the sampled-token gradient}
\label{app:weighted_gradient}

\begin{proof}[Derivation of \eqref{eq:weighted_opd_gradient}]
At a fixed context, let $p_v=p_\theta(v\mid s)$ and
$q_v=p^*(v\mid s)$. Differentiating with $w$ fixed gives
\begin{align*}
 \nabla_\theta\KL_{w(s,\cdot)}(p_\theta\|p^*)
 &=\sum_v w(s,v)\bigl[\log(p_v/q_v)+1-1\bigr]\nabla_\theta p_v
 \\
 &=\sum_v p_vw(s,v)\log(p_v/q_v)\nabla_\theta\log p_v.
\end{align*}
The derivative of $-p_v$ cancels the additional $1$ separately for
each token. At $\theta=\bar\theta$, the last sum is a conditional
expectation under the reference student's next-token distribution.
Taking the reference expectation over contexts therefore yields
\begin{equation*}
 -\nabla D_w(\bar\theta)
 =\EE\left[\sum_t w_t
       \log\frac{p_t^*}{p_t(\bar\theta)}
       \left.\nabla_\theta\log p_\theta(y_t\mid s_t)
       \right|_{\theta=\bar\theta}\right]
 =\EE\left[\sum_t w_tg_t\right].
\end{equation*}
The interchange is justified by Assumption~\ref{ass:local_regularity}.
\end{proof}

The identity allows token-dependent $w$ because the full conditional
expectation is retained. Differentiating the reference sampling law,
or recomputing $w$ as a function of the optimization variable, would
introduce additional terms.

An explicit unbiased sampled loss is
\begin{equation*}
 \widehat D_w(\theta;x,y)
 :=\sum_t \frac{w_t}{p_t(\bar\theta)}
       \left[p_t(\theta)\log\frac{p_t(\theta)}{p_t^*}
             -p_t(\theta)+p_t^*\right],
\end{equation*}
where $(x,y)$ is sampled from the reference policy and the denominator
and weights are fixed. Conditional expectation over $y_t$ recovers
the vocabulary sum in \eqref{eq:frozen_objective}. At $\bar\theta$,
the factor $p_t(\theta)/p_t(\bar\theta)$ in its gradient is one, so
$-\nabla\widehat D_w(\bar\theta;x,y)=\sum_t w_tg_t$.
This explains why \eqref{eq:weighted_opd_gradient} contains no explicit
importance denominator. At another parameter value, the probability
ratio must be retained when using the same reference rollouts.

For the interpretation of $g_t$, let
$z_t=\left.\nabla_\theta\log p_\theta(y_t\mid s_t)\right|_{\bar\theta}$.
Then
\begin{equation*}
 \left.\frac{\mathrm d}{\mathrm d\epsilon}
 \log p_{\bar\theta+\epsilon g_t}(y_t\mid s_t)
 \right|_{\epsilon=0}
 =\log\frac{p_t^*}{p_t(\bar\theta)}\,\|z_t\|^2.
\end{equation*}
Thus the infinitesimal correction increases or decreases the sampled
token's log-probability according to the teacher--student log-ratio,
unless $z_t=0$. This is a local statement about one token; descent of
the full distillation loss follows from the aggregate gradient.

\subsubsection{Descent of the lower-level objective}
\label{app:inner_step}

If, in addition, $D_w$ has an $L_D$-Lipschitz gradient on $U$,
uniformly over feasible $w$, its one-step response decreases the
fixed-reference loss whenever $\eta L_D\le1$.
\begin{proof}
The smoothness inequality and
$\theta^+(w)=\bar\theta-\eta\nabla D_w(\bar\theta)$ give
\begin{align*}
 D_w(\theta^+(w))
 &\le D_w(\bar\theta)
       -\eta\left(1-\frac{\eta L_D}{2}\right)
                      \|\nabla D_w(\bar\theta)\|^2\\
 &\le D_w(\bar\theta)-\frac{\eta}{2}
                      \|\nabla D_w(\bar\theta)\|^2.
\end{align*}
\end{proof}
This additional loss-smoothness assumption is only needed for the
descent statement, not for the reward bounds below.

\subsection{Bilevel formulation with multiple inner solutions}
\label{app:bilevel}

Let $\mathcal P$ denote the student parameter space. For a fixed
reference distribution, define
\begin{equation*}
 \Theta(w):=\argmin_{\theta\in\mathcal P}D_w(\theta).
\end{equation*}
The optimistic convention used in \eqref{eq:bilevel} selects the
highest-reward member of this set:
\begin{equation}
 \max_{w\in\mathcal W}\;\max_{\theta\in\Theta(w)}R(\theta).
 \label{eq:optimistic_bilevel}
\end{equation}
We assume the inner solution sets are nonempty and that the displayed
maxima are attained. If a reward maximum is not attained, the
corresponding optimal value is written as a supremum; an empty inner
solution set instead requires a different formulation.

If exact teacher matching is attainable, all feasible positive
weightings share the same zero-loss minimizers by
Theorem~\ref{thm:weighted_opd_validity}. Weighting can change the
best attainable compromise when exact matching is infeasible and the
learning trajectory before convergence. Our local analysis replaces
$\Theta(w)$ by the specified response $\theta^+(w)$. It therefore
does not require an attained or unique global inner optimum and does
not establish convergence to \eqref{eq:optimistic_bilevel}.

\subsection{Credit as sensitivity to the weight function}
\label{app:credit}

Let $F(w)=R(\theta^+(w))$ and define
$c(s,v)=g(s,v)^\top\nabla R(\theta_{\mathrm{vani}}^+)$.
The next calculation makes the derivative notation in
\eqref{eq:reward_credit} precise.

\begin{proof}[Derivation of the credit]
For a bounded measurable perturbation $h$, the function
$1+\epsilon h$ is feasible for sufficiently small $|\epsilon|$.
Linearity of the response gives
\[
 \theta^+(1+\epsilon h)
 =\theta_{\mathrm{vani}}^++\eta\epsilon\int hg\,\mathrm d\mu.
\]
Applying the chain rule,
\begin{equation}
 \left.\frac{\mathrm d}{\mathrm d\epsilon}F(1+\epsilon h)
 \right|_{\epsilon=0}
 =\eta\int h\,g^\top\nabla R(\theta_{\mathrm{vani}}^+)\,\mathrm d\mu
 =\eta\langle h,c\rangle_\mu.
 \label{eq:functional_credit}
\end{equation}
Moreover,
$\|c\|_\mu^2\le M\|\nabla R(\theta_{\mathrm{vani}}^+)\|^2<\infty$.
Thus the $L^2(\mu)$ gradient of $F$ at $w=1$ is $\eta c$.
\end{proof}

Accordingly, $\partial R(\theta^+(w))/\partial w_t$ in the main
text denotes sensitivity per unit reference measure, evaluated at
$(s_t,y_t)$. For an atom $(s,v)$, the ordinary coordinate derivative
instead includes its mass:
\begin{equation*}
 \left.\frac{\partial F(w)}{\partial w(s,v)}\right|_{w=1}
 =\eta\mu(\{(s,v)\})c(s,v).
\end{equation*}
A derivative with respect to a weight in a finite sample average
likewise includes that average's coefficient. Changing weights only
on a set of zero reference measure has no effect on $F$.

\subsection{Proof of Proposition~\ref{prop:reward_lower_bound}}
\label{app:lower_bound_proof}

\begin{proof}
The argument has two steps: bound the change in the student response,
then apply reward smoothness.

\noindent\emph{Step 1: Control the response change.}
Let $h=w-1$ and
$\Delta(w)=\theta^+(w)-\theta_{\mathrm{vani}}^+
=\eta\int hg\,\mathrm d\mu$. By Cauchy--Schwarz,
\begin{equation*}
 \|\Delta(w)\|^2
 \le\eta^2\left(\int h^2\,\mathrm d\mu\right)
            \left(\int\|g\|^2\,\mathrm d\mu\right)
 =\eta^2M\|h\|_\mu^2.
\end{equation*}
All terms are finite because $h$ is bounded, $\mu$ has finite mass,
and $M<\infty$. The trajectory-length contribution is already
included in $\mu$ and $M$.

\noindent\emph{Step 2: Lower-bound the reward change.}
Both responses and the segment joining them lie in $U$. Therefore,
\begin{align*}
 R(\theta^+(w))-R(\theta_{\mathrm{vani}}^+)
 &\ge\nabla R(\theta_{\mathrm{vani}}^+)^\top\Delta(w)
              -\frac{L_R}{2}\|\Delta(w)\|^2\\
 &\ge\eta\langle h,c\rangle_\mu
              -\frac{L_R\eta^2M}{2}\|h\|_\mu^2.
\end{align*}
Using $\lambda\eta L_RM\le1$ and the definition of $\mu$ gives
\[
 R(\theta^+(w))-R(\theta_{\mathrm{vani}}^+)
 \ge\eta\int\left(hc-\frac{h^2}{2\lambda}\right)\,\mathrm d\mu,
\]
which is \eqref{eq:reward_lower_bound}. The bound is exact at $w=1$,
and its first variation there agrees with \eqref{eq:functional_credit}.
\end{proof}

\subsection{Clipped weights and proof of Theorem~\ref{thm:reward_improvement}}
\label{app:reward_improvement}

\begin{proof}
We first optimize the lower bound pointwise, then quantify its gain.

\noindent\emph{Step 1: Maximize the surrogate.}
Completing the square yields
\[
 (w-1)c-\frac{(w-1)^2}{2\lambda}
 =\frac{\lambda c^2}{2}-\frac{(w-1-\lambda c)^2}{2\lambda}.
\]
Its unique maximizer on $[w_{\min},w_{\max}]$ is
$w^+=\clip(1+\lambda c,w_{\min},w_{\max})$.
This function is measurable and feasible. Since the surrogate is
integrable, pointwise maximization also maximizes its expectation,
uniquely up to sets of zero $\mu$-measure.

\noindent\emph{Step 2: Bound the improvement.}
Write $h^+=w^+-1$. The first-order optimality condition, evaluated
against the feasible choice $w=1$, gives
\begin{equation*}
 \left(c-\frac{h^+}{\lambda}\right)(1-w^+)\le0,
 \qquad h^+c\ge\frac{(h^+)^2}{\lambda}.
\end{equation*}
Integrating and applying Proposition~\ref{prop:reward_lower_bound},
\[
 R(\theta^+(w^+))-R(\theta_{\mathrm{vani}}^+)
 \ge\frac{\eta}{2\lambda}\|w^+-1\|_\mu^2\ge0.
\]
This is \eqref{eq:reward_improvement}. The gain is strictly positive
when $w^+\ne1$ on a set of positive $\mu$-measure. Because $1$ lies
strictly inside the feasible interval, this is equivalent to nonzero
credit on a set of positive measure. If the credit vanishes almost
everywhere, the update recovers vanilla OPD.
\end{proof}

\FloatBarrier

\section{Supplementary Experimental Results}
\label{app:supp}

This section provides additional experimental results and training dynamics.
All quantities below are measured on the student's own training rollouts.

Figure~\ref{fig:curve-reward-row} shows the training-reward dynamics for the experiments in Table \ref{tab:s2w}.

\begin{figure}[!ht]
\centering
\captionsetup[subfigure]{labelformat=empty}

\begin{minipage}{0.49\linewidth}
    \centering
    Instruct student: Qwen3-1.7B
\end{minipage}
\hfill
\begin{minipage}{0.49\linewidth}
    \centering
    Base student: Qwen3-1.7B-Base
\end{minipage}

\vspace{0.35em}

\begin{subfigure}[t]{0.24\linewidth}
    \centering
    \includegraphics[width=\linewidth]{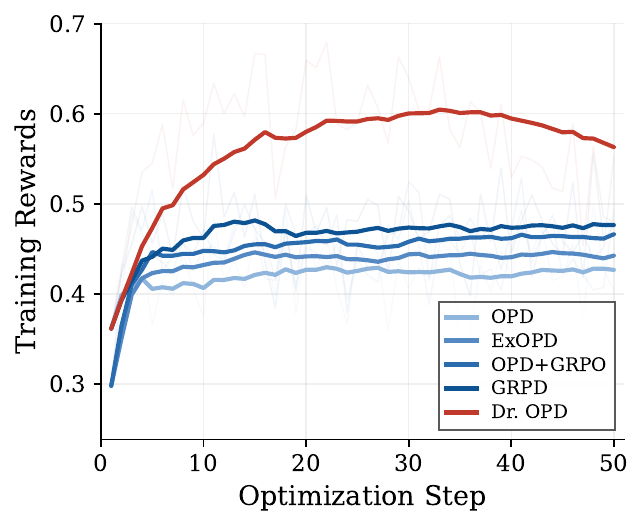}
    \caption{Math}
    \label{fig:curve-reward-math-1p7binstruct}
\end{subfigure}
\hfill
\begin{subfigure}[t]{0.24\linewidth}
    \centering
    \includegraphics[width=\linewidth]{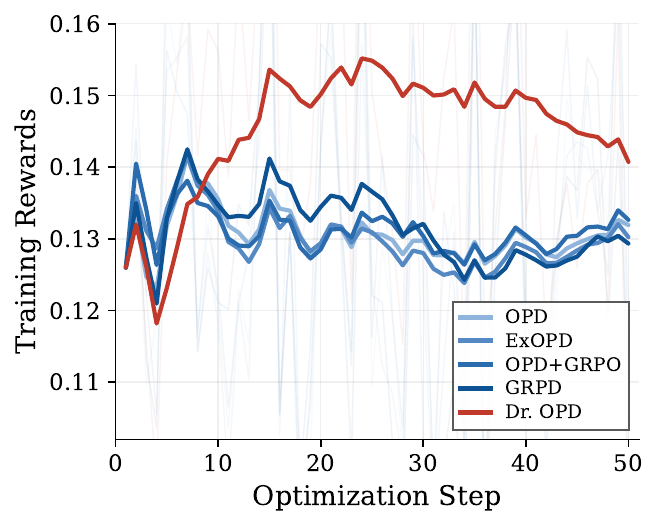}
    \caption{Code}
    \label{fig:curve-reward-code-1p7binstruct}
\end{subfigure}
\hfill
\begin{subfigure}[t]{0.24\linewidth}
    \centering
    \includegraphics[width=\linewidth]{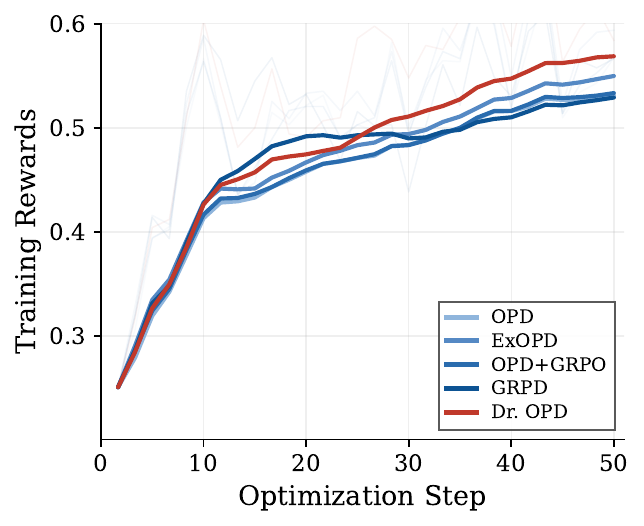}
    \caption{Math}
    \label{fig:curve-reward-math-1p7bbase}
\end{subfigure}
\hfill
\begin{subfigure}[t]{0.24\linewidth}
    \centering
    \includegraphics[width=\linewidth]{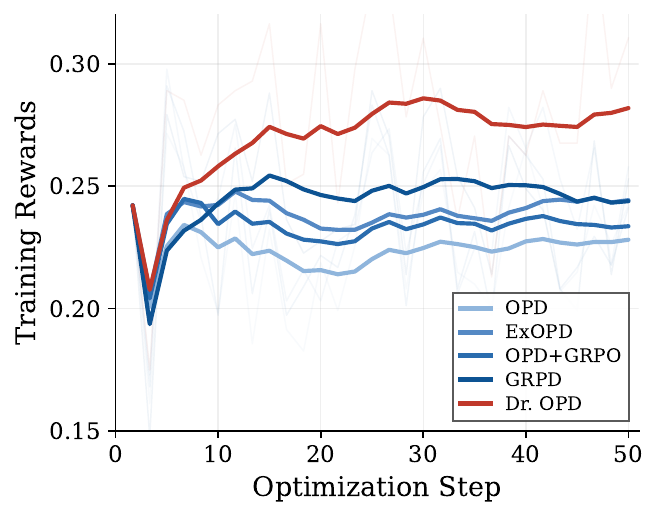}
    \caption{Code}
    \label{fig:curve-reward-code-1p7bbase}
\end{subfigure}

\caption{Dynamics of training rewards across students and tasks, computed by verifier accuracy of the student's own rollouts during training.
Lines are exponential moving averages ($\alpha=0.95$).
}
\label{fig:curve-reward-row}
\end{figure}

Figure~\ref{fig:curve-4bbase} shows the training dynamics for the
Qwen3-4B-Base student distilled from an RL-trained teacher of the same model
size. 
The training reward is the fraction of rollouts receiving a positive outcome
reward, and the response length is the average number of generated tokens per
rollout.
Although the student and teacher differ only in post-training, Dr. OPD
maintains a clear advantage in training reward.

\begin{figure}[!ht]
\centering

\makebox[\linewidth][c]{\begin{subfigure}[t]{0.25\linewidth}
  \centering
  \includegraphics[width=\linewidth]{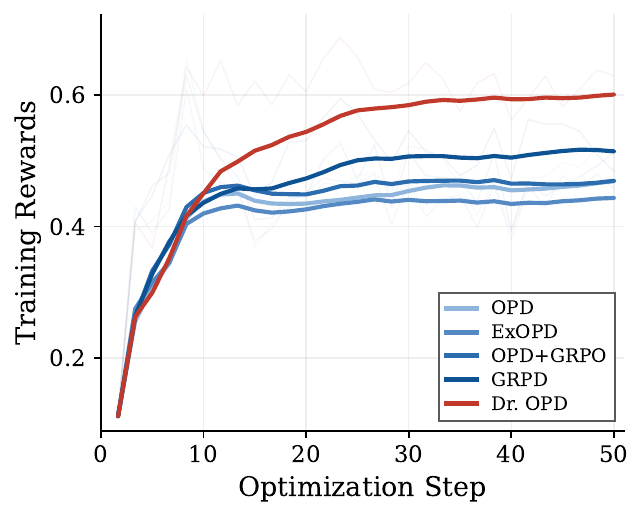}
  \caption{Training reward}
  \label{fig:curve-4bbase-reward}
\end{subfigure}
\hspace{0.03\linewidth}
\begin{subfigure}[t]{0.25\linewidth}
  \centering
  \includegraphics[width=\linewidth]{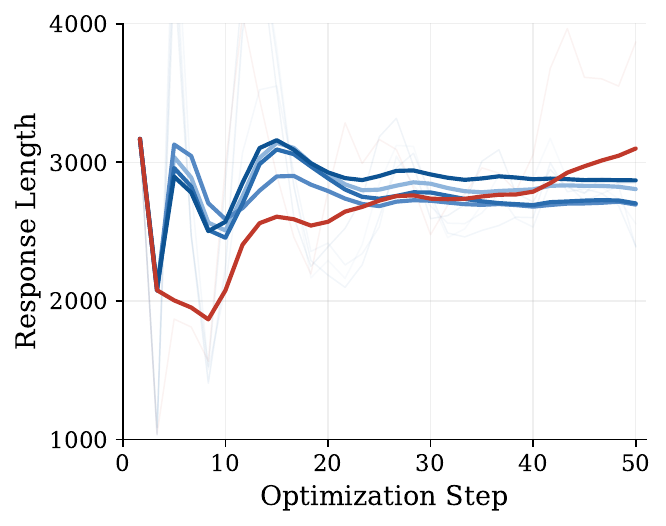}
  \caption{Response length}
  \label{fig:curve-4bbase-len}
\end{subfigure}}

\caption{Training dynamics for the Qwen3-4B-Base student distilled from the
Qwen3-4B-Base-RL teacher on math.
Thick lines show exponential moving averages ($\alpha=0.95$).
}
\label{fig:curve-4bbase}
\end{figure}

Figure~\ref{fig:lambda-stability} further examines the effect of the weighting
strength $\lambda$, with $\lambda=0$ corresponding to vanilla OPD. All tested
positive values of $\lambda$ achieve higher training rewards than vanilla OPD
and exhibit broadly similar reward curves. In contrast, response length can
vary substantially for larger values of $\lambda$, with differences of
thousands of tokens late in training. 
\begin{figure}[!ht]
\centering

\makebox[\linewidth][c]{\begin{subfigure}[t]{0.25\linewidth}
  \centering
  \includegraphics[width=\linewidth]{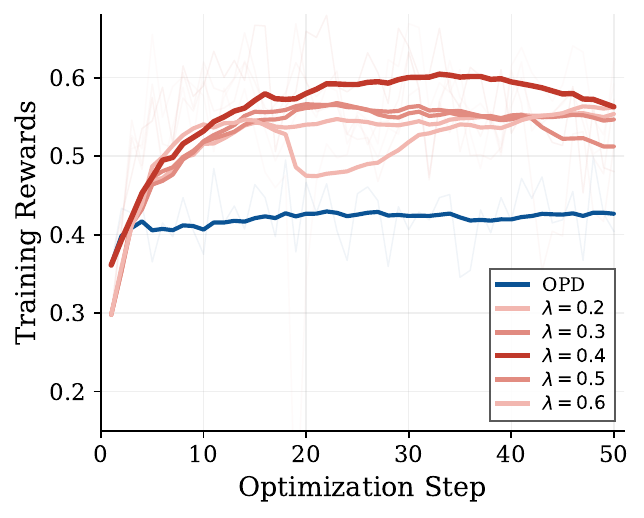}
  \caption{Training reward}
  \label{fig:curve-reward}
\end{subfigure}
\hspace{0.025\linewidth}
\begin{subfigure}[t]{0.25\linewidth}
  \centering
  \includegraphics[width=\linewidth]{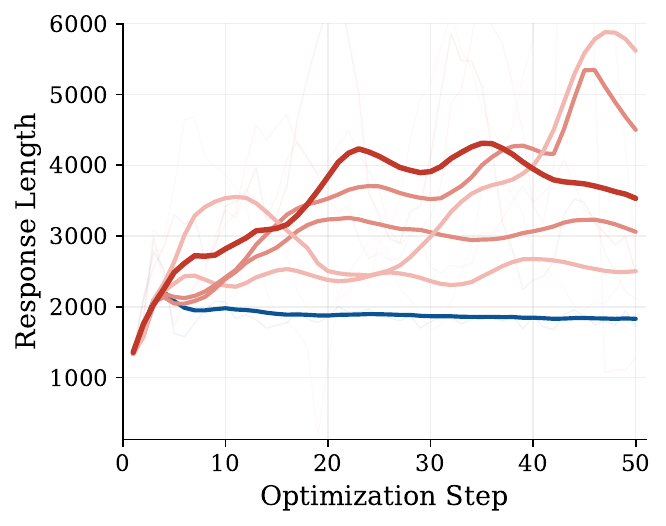}
  \caption{Response length}
  \label{fig:curve-len}
\end{subfigure}
\hspace{0.025\linewidth}
\begin{subfigure}[t]{0.25\linewidth}
  \centering
  \includegraphics[width=\linewidth]{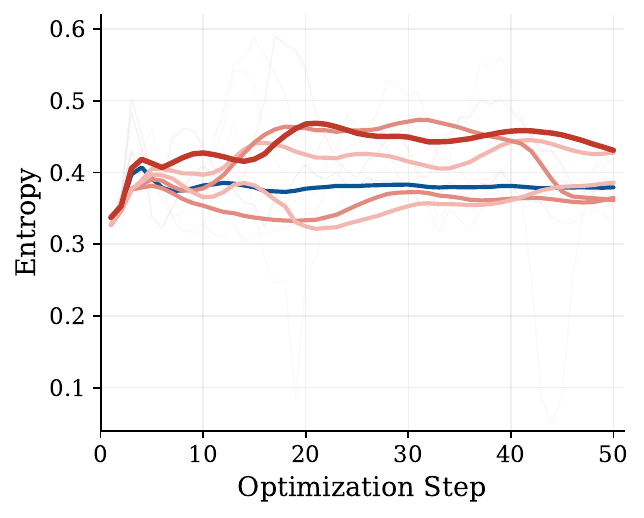}
  \caption{Entropy}
  \label{fig:curve-entropy}
\end{subfigure}}

\caption{Training dynamics under different values of $\lambda$.
(a)~Training reward, (b)~response length, and (c)~entropy over training
steps for vanilla OPD ($\lambda=0$) and Dr. OPD with
$\lambda\in\{0.2,0.3,0.4,0.5,0.6\}$, using the Qwen3-1.7B student and Qwen3-4B teacher. Thick lines show exponential moving averages ($\alpha=0.95$).
}
\label{fig:lambda-stability}
\end{figure}

\end{document}